\documentclass[conference]{IEEEtran}
\IEEEoverridecommandlockouts
\usepackage{cite}
\usepackage{amsmath,amssymb,amsfonts,amsthm}
\usepackage{graphicx}
\usepackage{textcomp}
\usepackage{xcolor}
\usepackage{mathtools}
\usepackage{bbm}
\usepackage{algorithm}
\usepackage[noend]{algpseudocode}%
\usepackage{cancel}
\usepackage{ulem}
\usepackage{tikz}
\usepackage{textcomp}
\usepackage{hyperref}

\newcommand{\etal}{\textit{et al.}}
\newtheorem{definition}{Definition}
\newtheorem{proposition}{Proposition}
\DeclareMathOperator*{\argmin}{arg\,min}

\def\BibTeX{{\rm B\kern-.05em{\sc i\kern-.025em b}\kern-.08em
    T\kern-.1667em\lower.7ex\hbox{E}\kern-.125emX}}
\begin{document}

\newcommand\copyrighttext{%
  \footnotesize \textcopyright 2021 IEEE. Personal use of this material is permitted.
  Permission from IEEE must be obtained for all other uses, in any current or future 
  media, including reprinting/republishing this material for advertising or promotional 
  purposes, creating new collective works, for resale or redistribution to servers or 
  lists, or reuse of any copyrighted component of this work in other works. 
 }
\newcommand\copyrightnotice{%
\begin{tikzpicture}[remember picture,overlay]
\node[anchor=south,yshift=10pt] at (current page.south) {\fbox{\parbox{\dimexpr\textwidth-\fboxsep-\fboxrule\relax}{\copyrighttext}}};
\end{tikzpicture}%
}

\title{DL-DDA - Deep Learning based Dynamic Difficulty Adjustment with UX and Gameplay constraints}

\author{\IEEEauthorblockN{Dvir Ben Or}
\IEEEauthorblockA{\textit{Playtika Research} \\
Herzeliya, Israel \\
dvirb@playtika.com}
\and
\IEEEauthorblockN{Michael Kolomenkin}
\IEEEauthorblockA{\textit{Playtika Research} \\
Herzeliya, Israel \\
michaelko@playtika.com}
\and
\IEEEauthorblockN{Gil Shabat}
\IEEEauthorblockA{\textit{Playtika Research} \\
Herzeliya, Israel \\
gils@playtika.com}
}

\maketitle
\copyrightnotice

\begin{abstract}
Dynamic difficulty adjustment ($DDA$) is a process of automatically changing a game difficulty for the optimization of user experience. It is a vital part of almost any modern game. Most existing DDA approaches concentrate on the experience of a player without looking at the rest of the players. We propose a method that automatically optimizes user experience while taking into consideration other players and macro constraints imposed by the game. The method is based on deep neural network architecture that involves a count loss constraint that has zero gradients in most of its support. We suggest a method to optimize this loss function and provide theoretical analysis for its performance. Finally, we provide empirical results of an internal experiment that was done on $200,000$ players and was found to outperform the corresponding manual heuristics crafted by game design experts.
\end{abstract}

\maketitle

\section{Introduction}

% What is DDA
Dynamic difficulty adjustment ($DDA$) is a process of automatically changing a game difficulty for the optimization of user experience. The difficulty of a game should be just right so that a player does not get bored when the game is too easy and does not get frustrated when the game is too hard. DDA is usually applied to each player based on the player’s abilities, skills and observed actions \cite{zohaib2018dynamic}. %The goal is to keep each player engaged from the beginning to the end regardless of the player's level.

% Why DDA is important
The inability of games to offer the right difficulty to everyone is considered one of the main reasons for players discontent~\cite{theoryoffun}. Players need a constant challenge to stay immersed in the game. 
% \giltodo{With the huge amount of games in the market, players do not think twice before leaving the current game for a competitor. Even a single moment of dissatisfaction with the game may turn out to be crucial for the game success.}

Recently, games have been moving from entertainment to other areas, such as healthcare~\cite{desmet2016participatory} and education~\cite{connolly2012systematic}. A well designed game is more than a way to spend free time and to relax. It might be a doctor's or a teacher's tool. Thus, the ability to dynamically adapt the game difficulty for each player is becoming even more important. 

% Why is DDA difficult - choose the right measure for engagement and incorporate economy constraints
DDA is a demanding task. Both industry and academia have been working on it for several decades, but it has not been completely resolved~\cite{zohaib2018dynamic}.
%\textit{\textcolor{red}{(I'm wondering if we could concat together the last two paragraphs, because the latter does not contain enough body to be left as is)}} 

There are two major challenges in devising a good DDA method. The first one is a precise formulation of user experience or user engagement. As stated in the first paragraph, DDA is a process that \textit{optimizes} user experience. We have to define the user experience first before we can optimize it. The definition should hold the properties of a loss function if we want to utilize  modern optimization tools~\cite{rosasco2004loss} and should be based on the data available in the game.

%The second challenge is the ability to incorporate game economy constraints. Keeping a player satisfied is an essential but insufficient condition. Games are a business. If the players are engaged, but are not paying, the business fails. Economy constraints are typically macro constraints on game parameters and results.

% was "third" challenge
The second challenge is interpretability and controllability. While DDA processes usually run automatically, games are managed by humans. The option to monitor and control the DDA output is vital for human operators.
%\textit{\textcolor{red}{(So,in the end are there two or three challenges? I see in the comments a second challenge that was ditched)}}

% \giltodo{Without such an option, any DDA method risks to remain an academic paper forever.}

% What we do to solve the problems
In this work we present a DDA system that addresses the aforementioned challenges. The system focuses on online games with many concurrent users. Specifically, the contribution of our paper is threefold:
\begin{enumerate}
    \item \textbf{UX loss function}. We introduce a novel formulation of user experience. As opposed to previous methods, our formulation leverages not only the experience of the player for whom the difficulty is computed, but also the experience of all other players. It requires that the player difficulty fits both the style of the player and the style of similar players. The formulation is generic and can be utilized in various games. We successfully employ the formulation as a loss function in a neural network that learns how to define the optimal difficulty based on the user state. %It is important to stress out, that the algorithm uses only data related to the playing style of the player, which does not violate the privacy of the player and meets the most strict privacy regulations.
    \item \textbf{Completion rate constraint}. We show how to use completion rate in a neural network. Completion rate is the percentage of players who finish a level or a task. It is often employed as an outside input to control the gameplay.
    
    Neural networks are the standard of modern machine learning. Thus, the ability to incorporate a mathematical constraint in a neural network is important for the application of the constraint to real life problems. Straightforward usage of the completion rate constraint in a neural network is difficult since the constraint is a piece-wise constant function and its gradient is zero almost everywhere. 
    
    We propose to use the completion rate in a variation of projection gradient descent algorithm~\cite{iusem2003convergence}. The algorithm projects the parameters of the neural network onto the feasible set defined by the completion rate constraint. We provide an alternation-based iterative procedure for the projection and give theoretical insights for the convergence of this procedure.
    
    \item \textbf{A real world DDA system}. Finally, we present a DDA system that was tested in an online game with millions of daily users. The system is based on a deep neural network that optimizes the UX loss function mentioned in Contribution (1) under the gameplay constraint (2). %The system also gives control and flexibility to the game designers. 
    We show that the system outperforms manually managed DDA methods.  
\end{enumerate}

The paper continues as following: Section~\ref{sect:relatedwork} depicts the related work. Section~\ref{sect:uxloss} describes the loss function. Section~\ref{sect:economy} explains how to integrate a common gameplay constraint - completion rate - in a neural network based solution. Section~\ref{sect:overview} outlines the general architecture and compares the results of our approach with a manual method. Section~\ref{sec:analysis} provides theoretical analysis including convergence analysis for the method.

\section{Related work}\label{sect:relatedwork}

% DDA is important
DDA has been an important research topic for the last several decades~\cite{zohaib2018dynamic}. The research can be roughly divided into three main groups.

% The first is purely theoretical, neuro science based
The first group searches for the optimal difficulty from the player physical responses. For example, Stein \etal~\cite{stein2018eeg} adjusted the difficulty according to the player EEG response and Wang \etal~\cite{wang2018adjusting}  used facial expressions to infer and adapt the experienced difficulty. While obtaining promising results, those approaches require special environments and are hardly applicable directly to existing games.

% The second group focuses on game states
The second group concentrates on game states. The methods of that group usually define an ideal number or order of states in the game and adjust the game parameters so that the order is preserved. For instance, Yannakakis and Hallam \cite{yannakakis2007towards} define \textit{the appropriate level of challenge} as the variance of steps required for the game engine to "kill" the player in predator-prey games. The higher the variance, the more interesting the game is. The variance is computed over a set of games. When the difficulty is too small, all the games last long. When the difficulty is large, the games end quickly. When the difficulty is right, some of the games end quickly and some last long. Xue \etal~\cite{xue2017dynamic} optimize the expected number of rounds in the game, while Sekhavat~\cite{sekhavat2017mprl} employs a similar approach, by optimizing the difference between the number of losses and the number of wins of a player in multiple periods. Hamlet system embedded in the \textit{Half-Life} game engine~\cite{hunicke2004,hunicke2005} assumes that the player should move between states according to the flow model. 
The system modifies the difficulty to increase the chance of relevant transitions, relying on observed statistics. Another approach is to maximize the speed of the player progress by a simulation reinforcement learning  mechanism~\cite{togelius2006making} and then apply it to real players.
The game state approaches strive to achieve uniform movement of players through the game. It is an appropriate choice for some games, but a disadvantage for other, where each player may want to advance on her own rate or where the optimal state flow is difficult to define.

% The third group, user based
The third group deals with player skills. The general idea is that better players should get harder games. The methods of that group predict player's abilities and performance and set the difficulty accordingly. For example, a system for Tower defence combines an estimation of player skills with the enemy potential~\cite{sutoyo2015dynamic}. Zook and Riedl~\cite{zook2012temporal} developed a method for predicting player performance in real time. A stroke rehabilitation system uses partially observable Markov model for estimating the player abilities~\cite{Goetschalckx2010GamesWD}.

The above approaches are capable of personalizing user experience, however they are game specific and do not provide a generic UX definition that can be applied to other games. 

% The main difference between us is the macro constraint and 
The approach presented here falls into the third category, but in addition to utilizing data of a single player, we exploit the data of all concurrent players. We verify that players with similar styles get similar difficulties. Moreover, to the best of our knowledge, our approach is the only one that offers a way to integrate a global gameplay constraint in the UX optimization process.

\section{UX loss function}
\label{sect:uxloss}
Loss functions are central part of any optimization system. It determines the error that the system minimizes. The proposed loss function is called UX loss function, since it optimizes the quality of user experience, or in other words, minimizes UX error. User experience is complex, since it depends on many factors, which are difficult to define \cite{mccarthy2004technology}. In this paper, the UX loss function is mostly focused on the difficulty and can be thought of as a ``DDA loss function". Yet, for general and theoretical reasons we continue with the term ``UX" throughout the paper. 

\subsection{Terminology and assumptions}

The goal is to set the difficulty $\hat{d_i}$ for each player $i$. We assume that the difficulty for all players is set for the same period of the game. Let's call the period $T$. The duration of the period can vary. The players do not have to participate in the game simultaneously.

We assume that there exists one-to-one mapping between the game difficulty and player performance and that the mapping is known. Technically, it means that we know a one-to-one function $Pd = f(d)$ that maps the difficulty $d$ to a game parameter $Pd$ measurable from the game data. For instance, the difficulty may correspond to the number of objects a player needs to find, amount of levels needed to pass, the strength of the opponent needed to fight or any combination of them. The assumption also means that the actual difficulty $d$ can be computed from the performance as $f^{-1}(Pd)$.
%\textit{\textcolor{red}{(Can't seem to understand the "metaphor" here. What does the "performance" refer to in our scenario?)}}

% The future performance can be predicted from the past performance similarly to any time series. Then the predicted difficulty is calculated as $f^{-1}()$ of the predicted performance.

% Note that the actual difficulty of the player is not always equal to the required game difficulty. For example, the player might have been required to find four objects, but actually found five. Thus, in this case the corresponding actual difficulty would be $d = f^{-1}(5)$.

% --- The below paragraph also appears in System section --- % 
In addition, we assume that players can be clustered into homogeneous groups. The details of the clustering are explained in Section~\ref{sect:overview}.
% In addition, we assume that players can be clustered into homogeneous groups. We cluster players with a K-Means algorithm~\cite{lloyd1982least}, but any other algorithm would do. The distance function for the  K-Means is a normalized euclidean distance between various player parameters, such as average game round duration, time from install, etc. The precise definition of similarity is unimportant as long as the players are divided in smaller groups with comparable properties.

\subsection{Loss definition}
The loss function combines two components. The first component ensures that the advancement in the game will fit the player personally. The idea is that if the required performance deviates too much from the actual performance, the player will find the game as too hard or too easy. This can be associated with positive experience.
The second component is that the requirements of resembling players should be similar. The intuition is that correctly designed user experience should not change much among players that are comparable to each other. This can be associated with game fairness and algorithm stability.
Mathematically, the loss function is defined as:
\begin{equation}\label{eq:userxpmin}
    \text{UX Loss}(\hat{\mathbf{D}}) = \text{var}(\hat{\mathbf{D}}) + \frac{\alpha}{M}\sum_{i=0}^{M} \left( d_i - \hat{d_i} \right)^2,
\end{equation}
where $d_i$ is the actual difficulty of player $i$, $\hat{d_i}$ is the difficulty we aim to find, $\hat{\mathbf{D}}$ is the set of required difficulties of all players in the cluster of player $i$:
\[
\hat{\mathbf{D}} = \{ \hat{d_0}, \hat{d_1}, \dots \hat{d_M}\},
\]
where $M$ is the size of the cluster and $\alpha$ is a parameter that controls the relative weight of the two parts of the loss function. In our experiments, we gave equal  weights to both parts, i.e. $\alpha=1$. Further investigation can be done in order to determine the affect of $\alpha$ on the actual user experience.

% Since the actual difficulty $d_i$ is not known before the relevant game period $T$ starts, and the computation of the required difficulties after the period $T$ is worthless, we replace the actual difficulties $d_i$ in Equation~\ref{eq:userxpmin} with the \textbf{predicted} difficulties. Although predictions are never equal to the actual values, we assume that they are accurate enough and our experiments justify that assumption. We use the same symbol $d_i$ for predicted difficulties to improve the readability of notations.
% \textit{\textcolor{red}{(That is not accurate. During training we do use future data for the optimization, i.e. the right part of the loss refers, as a label, to the observed behavior in the future, not to some model's output - this paragraph almost confused me, but i verified it in the code)}}

% In our experiments, we saw that the best results are obtained when the weights of both parts is the same, i.e. $\alpha=1$. However, to make the convergence faster, we update $\alpha$ on every cycle. It is set to the ratio of average values of both parts at the previous cycle:
% \begin{equation}
%     \alpha = \frac{\text{Average prev cycle}\left[\text{var}(\hat{\mathbf{D}})\right]}{\text{Average prev cycle} \left[ \sum_{i=0}^{M} \left( d_i - \hat{d_i} \right)^2 \right]}
% \end{equation}

\subsection{Loss Optimization}

% Conceptually, the optimization of Equation~\ref{eq:userxpmin} can be thought of as a two-step process. The first step is to predict the difficulties. The prediction can be done with a neural network. The second step is to optimize the UX Loss given the difficulties. The loss is a convex function and the optimization can be performed with any gradient based method.

Conceptually, the optimization of Equation~\ref{eq:userxpmin} can be thought of as a two-step process. The first step is to predict the actual difficulties $d_i$. This can be done with a neural network. The second step is to optimize the UX Loss given the difficulties. The loss is a convex function and the optimization can be performed with any gradient based method.

We use a single neural network for the two step process above. The UX Loss is used as the loss function of the network. The difficulties are not predicted explicitly. The network learns to set the required difficulties that minimize the loss given the actual difficulties provided in the training process and given the player behaviour of some period $T'$ before the relevant gaming period $T$. In our experiments, we set the duration of $T'$ to the duration of $T$. 
Formally, the network is defined as:
\begin{equation}\label{eq:optimaldiff}
    \hat{\mathbf{D}} = N(\boldsymbol{\Theta}, \mathbf{X}),
\end{equation}
where $N(\boldsymbol{\Theta}, \mathbf{X})$ represents the network, $\boldsymbol{\Theta}$ represents network parameters and $\mathbf{X} \in R^{MxZ}$ represents input features of dimension $Z$. The input features contain the states and the actions of a player in each day during $T'$.  

Algorithm~\ref{alg:projectionmin} summarizes the training procedure of the network:
\begin{algorithm}
    \caption{Train NN to minimize UX Loss}\label{alg:projectionmin}
    \begin{algorithmic}[1]
    \Require 
        \Statex $N(\boldsymbol{\Theta};\mathbf{X})$ - neural network, $\boldsymbol{\Theta}$ - initial network weights,
        $\mathbf{D} = \{d_0, d_1, \dots, d_M\}$ - difficulty per user during $T$, $\mathbf{X}$ - input features during $T'$.  
    \Ensure 
         \Statex $\hat{\mathbf{D}}$ - the set of required difficulties that optimizes Equation~\ref{eq:userxpmin}.
         \Statex $\hat{\boldsymbol{\Theta}}$ - optimized network weights for the required difficulties. 
    \State Apply a stochastic gradient descent to train the network.
    \State \Return $\boldsymbol{\Theta}$
    \end{algorithmic}
\end{algorithm}

At the inference time, the network computes the required difficulties from the input features. More details of the network appear in Section~\ref{sect:overview}.

\section{Completion rate constraint}\label{sect:economy}

In this section we show how to apply the completion rate constraint together with the neural network in Equation~\ref{eq:optimaldiff}.

\paragraph{Completion rate}

% What is Completion rate
Completion rate is a percentage of users that complete a certain goal or a series of goals in a game. The goal might be a task, a level, a mini sub-game or any game feature. Usually the completion rate is high for the first levels (easier ones) and gradually decreases with the advance in game levels. It can be thought of as a parameter that determines how challenging is a game feature.

% Why finishers rate is important
%Prestige is often proportional to the amount of effort that players are willing to put forth and, hence, to the amount of money they are willing to invest. Prestige and completion rate are often used by the product team as constraints for game features.

% Why combining finishers rate in NN is difficult problem
It seems natural to use completion rate as an optimization constraint for a DDA process in general and the UX loss function defined in Section~\ref{sect:uxloss} in particular. Indeed, it allows to optimize user experience while providing additional gameplay constraints. 

However, employing completion rate directly in an optimization is not trivial. Completion rate is a piece-wise constant function of the game difficulty and, thus, its gradient is zero everywhere except at a finite number of points. For instance, assume that when the difficulty is zero, the completion rate is one hundred percents, i.e. all players complete the given task. Increasing difficulty has no impact on the completion rate until the difficulty is high enough for one player to quit before completing. Then it has no effect again until the following player cannot complete and so on.

Optimization of functions with zero gradients is a complex problem, especially for neural networks. For some problems, it can be solved using Reinforcement Learning~\cite{xie2018environment}. For others, approximations and surrogate losses are used~\cite{grabocka2019learning}. Both solutions introduce their own problems.

% What we do
Instead, we suggest to use a variation of projected gradient descent. The idea is that the weights of a neural network can be projected onto the subspace where the completion rate constraint holds. The projection is performed at each iteration of the learning process. The projection itself is also an iterative procedure described below.

% How we solve the problem
\paragraph{Problem definition}

Let $P$ be the desired completion rate, $M$ be the number of players, $d_i$ is the actual difficulty (performance) from the training set and $\hat{d_i}$ is the desired difficulty (prediction) for player $i$ as defined in Equation~\ref{eq:userxpmin}. Then, the constraint is defined as:
%\begin{equation}\label{eq:completion}
%    \hat{P} - P = 0
%\end{equation}
\begin{equation}
\label{eq:completion}
    \hat{P} \triangleq \frac{1}{M}\sum_{i=1}^{M}{\mathbbm{1}[d_i \geq \hat{d_i}]} = P
\end{equation}
where $\mathbbm{1}[x]$ is an indicator function:
\[
\mathbbm{1}[x] \coloneqq \left\{ 
\begin{array}{ll} 
1 \; \; \text{ if } x \text{ is true} \\ 
0 \; \; \text{ if } x \text{ is false} \\
\end{array}
\right.
\]
and $\hat{P}$ is termed the \textit{achieved} completion rate, which is the number of players who were able to complete the challenge (a ``count" function) divided by the number of players.

% Recall that the predicted difficulty $d_i$ is defined for a period in the past. We do not know the future and have to use either data from the past or a prediction based on the past.

\paragraph{Projection}

The goal of the projection is to change the weights of the neural network $N(\boldsymbol{\Theta}, \mathbf{X})$ so that the Constraint~\ref{eq:completion} holds. The projection is performed during training after each time the neural network converges. After the projection, the neural network is trained again from the projection point.

The projection works by making the loss function of $N(\boldsymbol{\Theta}, \mathbf{X})$ roughly proportional to the absolute value of difference $\hat{P} - P$ in Constraint~\ref{eq:completion}. When the difference is positive, the achieved completion rate is higher than the desired completion rate. Hence, the desired difficulties $\hat{d_i}$ should raise. When the difference is negative, $\hat{d_i}$ should be lowered. 

Practically, we saw that good results are obtained when the loss function is equal to the average of the desired difficulties $\hat{d_i}$ at the previous iteration when $\hat{P} - P$ is positive, and to the minus average of $\hat{d_i}$ when it is negative, though other possibilities can be chosen, as long as the shift of the weights is done in the right direction, i.e. to increase or decrease the completion rate.

The average of $\hat{d_i}$ is not guaranteed to be proportional to $\|\hat{P}\|$. For example, only the required difficulty of a single player may raise all the time, keeping $\hat{P}$ constant while changing the average of the required difficulties. However, it is hardly possible in practice, since the neural network is already trained to compute all required difficulties. It is very difficult to change the network weights so that only a few difficulties will be influenced.

Algorithm~\ref{alg:projection_comp} summarizes the projection method. 
\begin{algorithm}[h]
    \caption{Projection to optimize completion rate}\label{alg:projection_comp}
    \begin{algorithmic}[1]
    \Require 
        \Statex $\eta$ - learning rate, $P$ - desired completion rate, $N(\boldsymbol{\Theta};\mathbf{X})$ - neural network, $\boldsymbol{\Theta}$ - initial network weights,
        $\mathbf{D} = \{d_0, d_1, \dots, d_M\}$ - actual difficulty during $T$
    \Ensure 
         \Statex $\hat{\boldsymbol{\Theta}}$ - optimized network weights that achieve $P$
    \Repeat
        \State // \textit{compute model outputs}
        \Statex \; \; \; $\hat{\mathbf{D}} = N(\boldsymbol{\Theta};\mathbf{X})$
        \State // \textit{compute hypothesized completion rate}
        \Statex \; \; \; $\hat{P} = \frac{1}{M}\sum_{i=1}^{M}{\mathbbm{1}[d_i \geq \hat{d_i}]}$ 
        \If{$\hat{P} < P$}
            \State // \textit{the computed rate is higher than desired}
            \Statex \; \; \; \; \; \: $err = \frac{1}{M}\sum_{i=1}^{M} \hat{d_i}$
        \ElsIf {$\hat{P} > P$}
            \State // \textit{the computed rate is lower than desired}
            \Statex \; \; \; \; \; \: $err = -\frac{1}{M}\sum_{i=1}^{M} \hat{d_i}$
        \Else {     }
            \State // \textit{The desired rate $P$ is reached}
            \Statex \; \; \; \; \; \;  \Return $ \boldsymbol{\Theta}$
        \EndIf
        \State // \textit{update model parameters}
        \Statex \; \; \;  $\boldsymbol{\Theta} \gets \boldsymbol{\Theta} - \eta \frac {\partial err}{\partial \boldsymbol{\Theta}}$
    \Until{max iterations or convergence}
        \State \Return $\boldsymbol{\Theta}$
    \end{algorithmic}
\end{algorithm}

\section{System overview}\label{sect:overview}

% Goal of the system and environment of the system
This section describes how the UX loss (Equation~\ref{eq:userxpmin}) and the completion rate constraint (Equation ~\ref{eq:completion}) are combined into a DDA system for online games with millions of daily users. The system was used in an internal experiment on a specific feature inside a game - and outperformed corresponding heuristic-based manual difficulty settings.

The DDA system consists of two stages. The first one, called \textit{Clustering} divides players into homogeneous groups. The second one creates an iterative mechanism for minimizing the UX loss and applying the completion rate constraint.

\paragraph{Clustering}

The variance part of Equation~\ref{eq:userxpmin} represents the user experience more accurately when the players resemble each other. In general, players may differ significantly. There are players who have just installed the game and there are players who have been in the game for several years. The players may have different tastes, preferences and gaming styles. It makes more sense require similar difficulties for similar users.

We assume that there exists a similarity function $S(p_i, p_j)$ between players $i$ and $j$. The function defines distance between players. The smaller the distance, the more similar the players are to each other. The purpose of the similarity function is to divide players into homogeneous clusters. 

We cluster players with a K-Means algorithm, but any other algorithm would do. We define similarity function as a normalized Euclidean distance between the input features $\mathbf{X}$ from Equation~\ref{eq:optimaldiff}. We note that the precise definition of similarity is unimportant as long as the players are divided in smaller groups with comparable properties.

\paragraph{Iterative optimization}

A single projection of the network weights as described in Section~\ref{sect:economy} is insufficient. When the projection is done, the weights $\boldsymbol{\Theta}$ are altered and the neural network no longer achieves the minimal error. It has to be retrained. The whole procedure is repeated until convergence.

% We define the convergence as the co-existence of two conditions. The first one is 

Algorithm~\ref{alg:flow} outlines the flow of the whole system.

\begin{algorithm}
    \caption{Full DDA system}\label{alg:flow}
    \begin{algorithmic}[1]
    \Require
        \Statex $K$ - Number of clusters,
        $P$ - Desired completion rate,
        $N(\boldsymbol{\Theta};\mathbf{X})$ - neural network architecture, $\mathbf{D} = \{d_0, d_1, \dots, d_M\}$ - actual difficulty at $T$
    \Ensure 
        \Statex $\{ \hat{\mathbf{D}}_k \}_{k=1}^K$ - the required difficulty for every player of cluster $k$
        \State Initialize neural network parameters $\boldsymbol{\Theta}_k$ with Xavier \cite{glorot2010understanding}
    \For {$k$ in range$(K)$}
    \State assign $\mathbf{X}$ with the $k^{\text{th}}$ cluster features data set
    \State assign $\mathbf{D}$ with the corresponding actual difficulty
    \Repeat {           } // \textit{alternation cycles}
        \State \label{algstage:uxloss} // \textit{optimize UX loss}
        \Statex \; \; \: \; \; \; update $\boldsymbol{\Theta}_k$ by applying $N(\boldsymbol{\Theta}_k;\mathbf{X})$ to optimize Eq.~\ref{eq:userxpmin}
        \State \label{algstage:completion} // \textit{Project weight to ensure completion rate}
        \Statex \; \; \: \; \; \; update $\boldsymbol{\Theta}_k$ by applying algorithm~\ref{alg:projection_comp}
        %     \State $t \gets t+1$
    \Until{max iterations or convergence}
    \State compute $\hat{\mathbf{D}}_k = N(\boldsymbol{\Theta};\mathbf{X})$
    \EndFor
    \end{algorithmic}
\end{algorithm}

\begin{figure}[ht]
  \centering
    \includegraphics[width=0.47\textwidth]{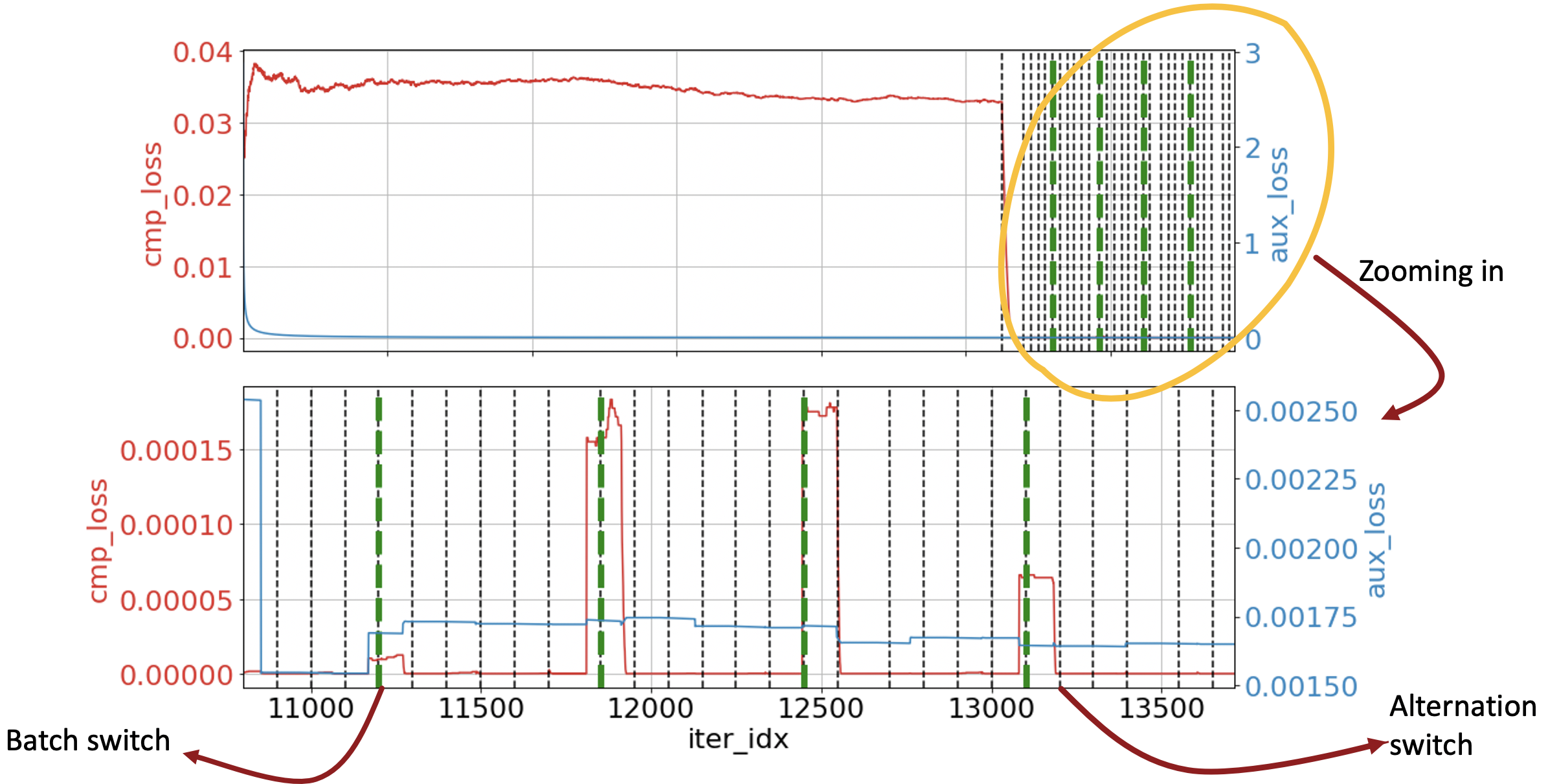}
    \caption{Training procedure illustration. Horizontal axis corresponds to progressing training steps of the model. The red curve (corresponds to left vertical axis) displays the $l_2$ distance between the evaluated completion rate and the desired one, while the blue curve (corresponds to right vertical axis) describes the loss formulated in Eq.~\ref{eq:userxpmin}. Dashed black vertical lines indicate the switch from UX Loss to Completion rate projection and vice versa. The green bold vertical lines indicate the replacement of batch samples used for optimization. For convenience, the lower chart provides a closer view of the optimization trajectory described in the upper one, focused on the stage when the optimized objectives converge.}
    \label{fig:convergence}
\end{figure}

Figure~\ref{fig:convergence} provides an illustration of training procedure and the convergence of our approach for a single cluster. The longest part of the training procedure is the first optimization cycle in which the UX loss is minimized. It can be seen from the lower, zoomed-in part of the Figure, that the system converges both when the batch values change and when the loss switches from UX to projection.

\subsection{Implementation details}

We use a fully connected neural network with 5 hidden layers. The dimension of the input is 40. The input vector has a variety of aggregated player parameters for a two week period before a small mini-game optimized using this approach.

We used 200 clusters and required that the minimal number of players in a cluster is 5,000. The system runs on NVIDIA DGX computer. The whole process for a million of users takes around half an hour.

\subsection{Results}
% What we did
We performed an A/B test to verify the validity of our approach. The output was compared to our system with the difficulty levels set by a rule based method currently used by game operators. The rule based system is a result of several years of trial and error. It is a collection of \textit{if-else} decisions applied to a variety of game parameters. It incorporates a great deal of knowledge about the game and generally provides satisfactory results.

Our approach outperformed the rule based method as we show below. In addition to being superior in accuracy, our approach is automatic and saves time for game operators. Each change in game mechanism requires manual adaptation of the rule based system. The manual process is time consuming and prone to errors. 

The test was carried out on an eight day mini game (a feature inside one of Playtika's games) where the goal was to optimize the number of points each player has to obtain. About $800,000$ players were in the control group and received rule-based difficulties, while about $200,000$ players were in the test group receiving machine learning-based difficulties.

\begin{table}[h]
\centering
\begin{tabular}{ | c | c | c | }
\hline
 Target & Rule based & Our approach \\ 
 \hline
 8-10\% & 12.0\% & 8.7\% \\
 \hline
\end{tabular}\\
\caption{Comparison of the target completion rate with the result achieved by the rule based method and our approach}\label{tbl:flow}
\end{table}

Table~\ref{tbl:flow} compares the average completion rate achieved by our approach and the rule based method with the target range defined by the product team. The result of the rule based method is fine by the practical standards of the game, but our approach still outperforms it.

\begin{figure}[ht]
  \centering
    \includegraphics[width=0.45\textwidth]{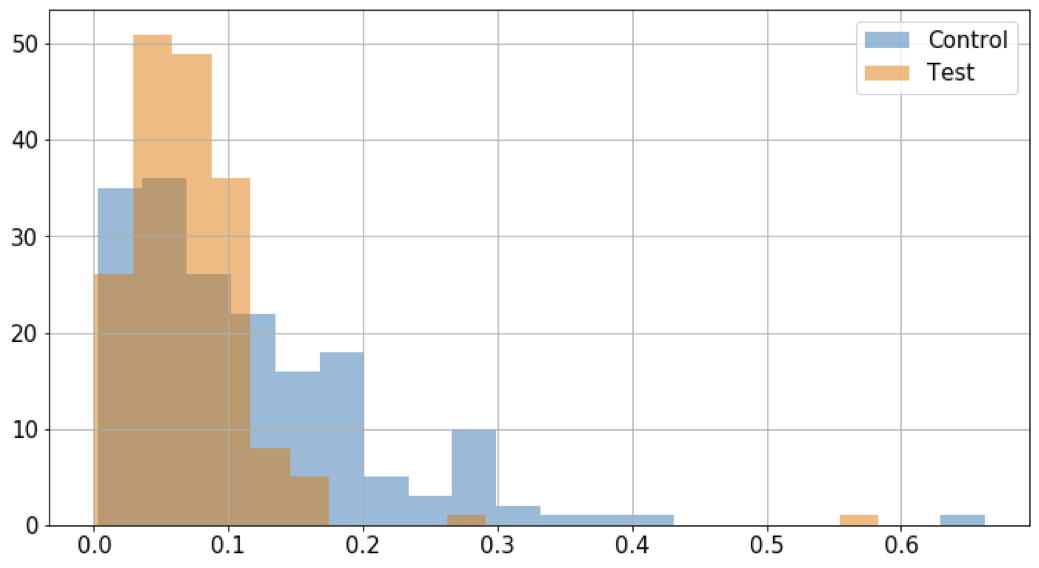}
    \caption{A histogram of completion rate by clusters. X-axis is the completion rate, Y-axis is the number of clusters that achieved the completion rate. Blue is the control group - rule based method, orange is the test group - our system.}
    \label{fig:clusters}
\end{figure}

The difference between the methods is even more pronounced when we look at the distribution of the completion rates. This is where our approach really shines. Figure~\ref{fig:clusters} shows the histograms of the completion rate by clusters. Recall that the clusters are the homogeneous groups of players computed by K-Means algorithm. The variance of the completion rate of the control group is much higher than that of the test group.

There are 200 clusters. The number 200 was chosen since it created homogeneous clusters on one hand and yet contained a relatively large numbers of players (around $1,000$) on the other hand.  In the control group 49 clusters had completion rate higher than 16\%. In the test group, only 5 clusters had the completion rate higher than 16\%. 

The meaning is that while the rule based method achieves satisfactory average completion rate, it fails to achieve it for a large proportion of population. The ability to control the completion rate of every sub group of players is another advantage of our approach.

% \section{Theoretical Analysis}
% % Introduction - what we do
% This section describes the condition for the convergence of the alternations of the Algorithm~\ref{alg:flow}. Recall that according to our definition the algorithm converges, when the distance  

% distance between the weights 

% $\boldsymbol{\Theta}_k$

% \label{algstage:completion}

% the convergence was defined as the  \mytodo{Don't forget to mention after the algorithm what we mean by convergence.}

\section{Theoretical Analysis}
\label{sec:analysis}
This section describes the condition for the convergence of the algorithm and provides some theoretical insights and in a sense the derivation is a bit similar to \cite{shabat2012interest}. The convergence does not assume convexity, but it does assume certain properties for the non-linear projection operators. Those properties depend on the function determined by the structure of the neural net and its loss function. 
The first projection operator returns the nearest local minimum point. 
\begin{definition}
\label{def:ProjMin}
Given a training dataset $\mathbf{X}$, 
a neural net $N(\boldsymbol{\Theta}, \mathbf{X})$ with weights $\boldsymbol{\Theta}$ and a set of local minima $\mathcal{M}$, then  $\mathcal{P_M} N(\boldsymbol{\Theta}, \mathbf{X})$ returns the closest weights of the nearest local minimum:

\begin{equation}
    \mathcal{P_M} N(\boldsymbol{\Theta}, \mathbf{X}) =  \argmin_{N(\boldsymbol{\hat{\Theta}}, \mathbf{X}) \in \mathcal{M}} \Vert \Theta - \hat{\Theta} \Vert_2
\end{equation}

\end{definition}
The second operator returns the closest set of weights, that satisfy the completion rate. Based on Eq. \ref{eq:completion}, it is possible to define the set of valid solutions.

\begin{definition}
\label{def:SetComp}
Let $\mathcal{C}$ be the set of all possible weights that given a set of predicted difficulties $\{d_i\}$, desired completion rate $P$ and tolerance $\delta$ such that
\begin{equation}
\mathcal{C} = \Bigl\{\boldsymbol{\Theta} \vert ~~~ \vert  \frac{1}{M}\sum_{i=1}^{M}{\mathbbm{1}[d_i \geq N(\boldsymbol{\Theta}, \mathbf{X})]} - P \vert \le \delta\Bigr\}  
\end{equation}
\end{definition}

\begin{definition}
\label{def:ProjComp}
Given a training dataset $\mathbf{X}$, 
a neural net $N(\boldsymbol{\Theta}, \mathbf{X})$ with weights $\boldsymbol{\Theta}$ and a set of valid completion points $\mathcal{C}$ (Definition \ref{def:SetComp}), then $\mathcal{P_C} N(\boldsymbol{\Theta}, \mathbf{X})$ returns the closest weights in $\mathcal{C}$:

\begin{equation}
    \mathcal{P_C} N(\boldsymbol{\Theta}, \mathbf{X}) =  \argmin_{N(\boldsymbol{\hat{\Theta}}, \mathbf{X}) \in \mathcal{C}} \Vert \Theta - \hat{\Theta} \Vert_2
\end{equation}

\end{definition}
The operators $\mathcal{P_M}$ and $\mathcal{P_C}$ are approximately implemented by Algorithm \ref{alg:projectionmin} and Algorithm \ref{alg:projection_comp}, respectively. The algorithms return a local minimum ($\mathcal{P_M}$) or a completion-valid point ($\mathcal{P_C}$) by the application of a stochastic gradient descent (or other optimizer), but does not guarantee to return the closest point, since it depends on the structure of the neural network, which is typically high-dimensional non-convnex manifold. This is different than the case in \cite{shabat2012interest}.

Optimizing the UX loss under the completion rate constraint, can be done by the following alternating scheme, which is approximately implemented by Algorithm \ref{alg:flow}:
\begin{equation}
    \label{eq:alt1}
   \Theta_i^M \leftarrow \mathcal{P_M}\Theta_i^C
\end{equation}

\begin{equation}
    \label{eq:alt2}
    \Theta_{i+1}^C \leftarrow \mathcal{P_C}\Theta_i^M
\end{equation}
starting from an arbitrary point (random initialization of the weights).

%\begin{lemma}
%Each application of an alteration, reduces the distance between two points. Formally, given $x \in \mathcal{C}$ and $y \in \mathcal{M}$, then for the application of $\mathcal{P_C}$
%\begin{equation}
%    \Vert \mathcal{P_C}y - y \Vert \le \Vert x - y \Vert
%\end{equation}
%and for the application of $\mathcal{P_M}$
%\end{lemma}

\begin{proposition}
\label{prop:conv}
Let $\Theta_i^C$ and $\Theta_i^M$ ($i \ge 1$) be a set of points (weights) obtained by a consecutive application of the alternation scheme (Eqs. \ref{eq:alt1} and \ref{eq:alt2}) then the series $\Vert \Theta_i^C - \Theta_i^M \Vert$ converges.
\end{proposition}

\begin{proof}
Since $i \ge 1$, then according to Eq. $\ref{eq:alt2}$, $\Theta_i^C \in \mathcal{C}$ (Def. \ref{def:SetComp}). By the definition of $\mathcal{P_M}$, $\Theta_i^M$ is the closest local minima to $\Theta_i^C$ and by the definition of $\mathcal{P_C}$, $\Theta_{i+1}^C$ is the closest valid count-constraint point to $\Theta_i^M$. Since $\Theta_i^C \in \mathcal{C}$ and $\Theta_{i+1}^C \in \mathcal{C}$ is the closest point to $\Theta_i^M$
\begin{equation}
\label{eq:propeq1}
    \Vert \Theta_i^M - \Theta_{i+1}^C \Vert \le \Vert \Theta_i^M - \Theta_i^C \Vert.
\end{equation}
By the definition of $\mathcal{P_M}$, $\Theta_{i+1}^M \in \mathcal{M}$ is the closest local minima to $\Theta_{i+1}^C$. Since $\Theta_i^M \in \mathcal{M}$
\begin{equation}
\label{eq:propeq2}
    \Vert \Theta_{i+1}^C - \Theta_{i+1}^M \Vert \le \Vert \Theta_i^M - \Theta_{i+1}^C \Vert
\end{equation}
Combining Eqs. \ref{eq:propeq1} and \ref{eq:propeq2} gives
\begin{equation*}
    \Vert \Theta_{i+1}^M - \Theta_{i+1}^C \Vert \le \Vert \Theta_i^M - \Theta_i^C \Vert.
\end{equation*}
Since $\Vert \Theta_i^M - \Theta_i^C \Vert$ is monotonically decreasing and bounded it converges, which completes the proof.
\end{proof}
Proposition \ref{prop:conv} states that the distance between a valid completion rate point and a local minima point is monotonically decreasing and eventually converges. An interesting observation from the proposition is that it tells us where to look for the next minima/valid completion rate point, which enables to decrease the step size of the SGD proportionally to the distance between the two points.

\begin{figure}[h]
  \centering
    \includegraphics[width=0.47\textwidth]{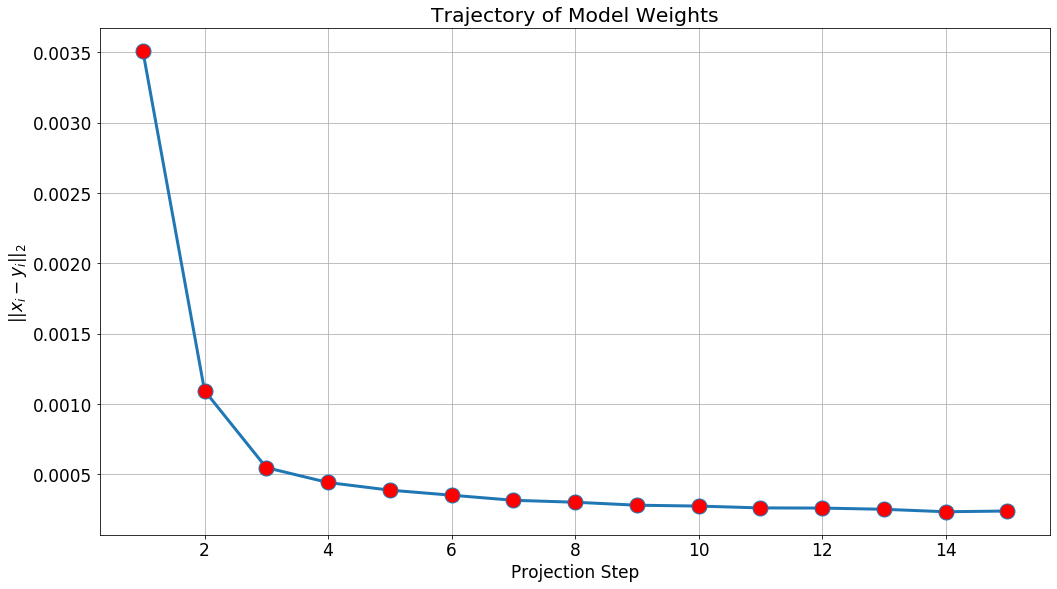}
    \caption{The distance between the weights of a valid completion rate point to the local minima followed by the application of $\mathcal{P_C}$ to it}
    \label{fig:moon_conv}
\end{figure}

Figure \ref{fig:moon_conv} depicts the convergence of Algorithm \ref{alg:flow} to illustrate Proposition \ref{prop:conv} on real data. The figure shows that the distance between a local minima and its corresponding completion rate valid point decreases with each iteration. Interestingly, the algorithm achieves a monotonically decreasing curve even for the last iterations, when the distance is small. This happens even when the projection operators are only approximated and do not satisfy the strict requirement of their definition for finding the closest point.

Additionally, the following observations infer directly from Proposition \ref{prop:conv}:
\begin{itemize}
    \item Since the distance between a valid completion point and a local minimum converges, then eventually it means (excluding pathological cases of points having \emph{exactly} the same distance) that the algorithm iterates between one local minimum and one valid completion rate point. Therefore, it converges to a specific local minimum/completion rate point.
    \item The difference in model's performance between those two points, depends on the distance and the Lipschitz constant of the neural network \cite{fazlyab2019efficient}. So if the distance is small (and hopefully the Lipschitz constant), then stopping in completion rate point or in a local minimum should not make a big difference.
\end{itemize}
For more details the reader is referred to \cite{or2020generalized}.
\section*{Conclusion}

% Method - summary of the paper
This paper presents a system for dynamic adjustment of game difficulty. The system was tried on an online game with millions of daily users and significantly outperformed manual heuristics used by the game developers. The system is based on several innovations. First, it presents a formulation of user experience that depends on all similar players. The formulation exploits more information than the existing methods.

Second, it shows how to incorporate a completion rate constraint in a neural network. The completion rate constraint is important for creating a fun experience. Straightforward application of the constraint to a neural network is difficult, since the gradients of the constraint are piece-wise zero. The paper also presents a theoretical analysis of the convergence of the neural network.

% Flexibility of the approach
While the system was applied to a specific game, it is very flexible and can easily be adapted in other games. All that one needs in order to employ our system is a definition of difficulty that can be computed from the game data, such as objects collected, monsters slayed, etc., and a definition of a similarity between players. Then the system can learn the appropriate difficulties for any required period of time.

% Disadvantages and the next steps
A possible drawback of our approach is that it does not allow to change the difficulties during the predefined period. If the chosen difficulty is too hard and the player does not advance in the game, it will stay hard. We plan to research how to incorporate real time information into our approach.

\bibliographystyle{IEEEtran}
\bibliography{biblio}

\end{document}